\documentclass{article}
\PassOptionsToPackage{numbers,compress}{natbib}

\usepackage{zelda_arxiv}
\usepackage[utf8]{inputenc}
\usepackage{natbib}
\usepackage{url}
\usepackage{subcaption}
\usepackage{hyperref}
\usepackage{xcolor}
\usepackage{amsthm,amsfonts,amsmath,amssymb,mathtools}
\usepackage{microtype}
\usepackage{hyperref}
\usepackage{xspace}
\usepackage{wrapfig}
\usepackage{todonotes}
\usepackage{enumitem}
\usepackage{booktabs}
\usepackage[noend]{algpseudocode}
\usepackage[capitalize]{cleveref}  
\usepackage{algorithm}

\newtheorem*{rep@theorem}{\rep@title}
\newcommand{\newreptheorem}[2]{%
\newenvironment{rep#1}[1]{%
 \def\rep@title{#2 \ref{##1}}%
 \begin{rep@theorem}}%
 {\end{rep@theorem}}}
\makeatother
\newreptheorem{theorem}{Theorem}
\newreptheorem{proposition}{Proposition}
\newreptheorem{lemma}{Lemma}
\newreptheorem{corollary}{Corollary}

\usepackage{etoolbox}
\usepackage{tikz}
\usetikzlibrary{tikzmark}
\usetikzlibrary{calc}

\errorcontextlines\maxdimen

\newcommand{\ALGtikzmarkcolor}{black}
\newcommand{\ALGtikzmarkextraindent}{4pt}
\newcommand{\ALGtikzmarkverticaloffsetstart}{-1ex}
\newcommand{\ALGtikzmarkverticaloffsetend}{-.5ex}
\makeatletter
\newcounter{ALG@tikzmark@tempcnta}

\newcommand\ALG@tikzmark@start{%
    \global\let\ALG@tikzmark@last\ALG@tikzmark@starttext%
    \expandafter\edef\csname ALG@tikzmark@\theALG@nested\endcsname{\theALG@tikzmark@tempcnta}%
    \tikzmark{ALG@tikzmark@start@\csname ALG@tikzmark@\theALG@nested\endcsname}%
    \addtocounter{ALG@tikzmark@tempcnta}{1}%
}

\def\ALG@tikzmark@starttext{start}
\newcommand\ALG@tikzmark@end{%
    \ifx\ALG@tikzmark@last\ALG@tikzmark@starttext
    \else
        \tikzmark{ALG@tikzmark@end@\csname ALG@tikzmark@\theALG@nested\endcsname}%
        \tikz[overlay,remember picture] \draw[\ALGtikzmarkcolor] let \p{S}=($(pic cs:ALG@tikzmark@start@\csname ALG@tikzmark@\theALG@nested\endcsname)+(\ALGtikzmarkextraindent,\ALGtikzmarkverticaloffsetstart)$), \p{E}=($(pic cs:ALG@tikzmark@end@\csname ALG@tikzmark@\theALG@nested\endcsname)+(\ALGtikzmarkextraindent,\ALGtikzmarkverticaloffsetend)$) in (\x{S},\y{S})--(\x{S},\y{E});%
    \fi
    \gdef\ALG@tikzmark@last{end}%
}

\apptocmd{\ALG@beginblock}{\ALG@tikzmark@start}{}{\errmessage{failed to patch}}
\pretocmd{\ALG@endblock}{\ALG@tikzmark@end}{}{\errmessage{failed to patch}}
\makeatother

\usepackage{thm-restate}
\crefname{section}{Sec.}{Secs.}
\crefname{lemma}{Lem.}{Lems.}
\crefname{proposition}{Prop.}{Props.}
\crefname{corollary}{Cor.}{Cors.}
\crefname{theorem}{Thm.}{Thms.}
\crefname{example}{Ex.}{Exs.}
\crefname{assumption}{Assump.}{Assumps.}
\crefname{equation}{Eq.}{Eqs.}
\crefname{definition}{Def.}{Defs.}
\crefname{appendix}{App.}{Apps.}

\definecolor{darkblue}{rgb}{0.0,0.0,0.55}
\hypersetup{  
  colorlinks = true,
  citecolor  = darkblue,
  linkcolor  = darkblue,
  citecolor  = darkblue,
  filecolor  = darkblue,
  urlcolor   = darkblue,
}

\newtheorem{proposition}{Proposition}[section]

\theoremstyle{definition}
\newtheorem{remark}{Remark}
\newtheorem{example}{Example}
\newtheorem{definition}{Definition}[section]

\DeclareMathOperator*{\argmin}{argmin}

\DeclarePairedDelimiterX{\bdx}[2]{[}{]}{%
  #1\;\delimsize\|\;#2%
}
\newcommand{\bd}{D\bdx}

\newcommand{\E}{\mathbb E}
\newcommand{\V}{\mathbb V}

\newcommand{\R}{\mathbb R}
\newcommand{\deq}{:=}
\newcommand{\dm}{\mathcal E}
\newcommand{\dv}{\mathcal V}
\newcommand{\iid}{\textit{i.i.d.}\xspace}

\newcommand{\eg}{\textit{e.g.}\xspace}

\newcommand{\nlsum}{\sum\nolimits}

\newcommand{\kl}{\textup{KL}}

\newcommand{\scal}[2]{\langle #1, #2\rangle}

\newcommand{\cala}{\mathcal{A}}

\newcommand{\eq}[1]{\begin{equation}#1\end{equation}}

\newcommand{\p}[1]{({#1})}

\newcommand{\norm}[1]{\lVert #1 \rVert}

\title{Understanding the bias-variance tradeoff \\of Bregman divergences}
\author{
        \name\!\!\! Ben Adlam\thanks{Equal contribution.} \email{adlam@google.com}\\
        \name Neha Gupta$^*$\thanks{Work done at Google.} \email{nehagupta@cs.stanford.edu}\\
        \name Zelda Mariet$^*$ \email{zmariet@google.com}\\
        \name Jamie Smith$^*$ \email{jamieas@google.com}\\
}

\date{}
\usepackage{graphicx}
\begin{document}

\maketitle

\begin{abstract}
    This paper builds upon the work of~\citet{generalized_bvd}, which generalized the bias variance tradeoff to any Bregman divergence loss function. \citet{generalized_bvd} showed that for Bregman divergences, the bias and variances are defined with respect to a \emph{central label}, defined as the mean of the label variable, and a \emph{central prediction}, of a more complex form. We show that, similarly to the label, the central prediction can be interpreted as the mean of a random variable, where the mean operates in a dual space defined by the loss function itself.  Viewing the bias-variance tradeoff through operations taken in dual space, we subsequently derive several results of interest. In particular, (a) the variance terms satisfy a generalized law of total variance; (b) if a source of randomness cannot be controlled, its contribution to the bias and variance has a closed form; (c) there exist natural ensembling operations in the label and prediction spaces which reduce the variance and do not affect the bias.
    \vspace{-.5em}
\end{abstract}

\section{Introduction}
\vspace{-.5em}
In machine learning, an algorithm $\cala$ uses data sampled from some unknown distribution $\mu$ to learn an approximation of $\mu$. For supervised learning, we assume the samples from $\mu$ are pairs $(X_i, Y_i)$ drawn \iid from $\mu$. The input-label pairs $(X_i, Y_i)$ form the \emph{training set} $T$ used by $\mathcal A$ to produce a predictive function $f$ that approximates the conditional distribution $(Y \mid X)$ of a label $Y$ given an input $X$.

Components of $Y$ may be unpredictable from the information in $X$, even with a perfect predictor. This results in the \emph{Bayes}, or irreducible, error (see, \eg,~\citep{hastie01statisticallearning}). However, even ignoring the Bayes error, algorithms rarely learn the conditional distribution $(Y\mid X)$ exactly. Thus, we must introduce a loss function $L$ to measure the resulting mistakes. Often, the loss $L$ is used by the algorithm $\cala$ to select the parameters of the model $f$. 

Mathematically, the predictive function $f$ returned by $\cala$ is a stochastic process that depends on the training set $T$.\footnote{If $\cala$ is a randomized algorithm (neural network,  random forest...), it may depend on additional random variables.} Thus, we measure how close the random variable $f(X; T)$ is to the target random variable $(Y|X)$ as
\eq{ \label{eq_loss_at_X}
    \E_{Y} [L(Y, f(X; T)) \mid X,T].
}
Often, rather than the punctual loss for an input $X$, we care about the average performance over a test set:
\eq{ \label{eq_loss}
    \E_{X,Y} [L(Y, f(X; T)) \mid T].
}
Typically, the expectation in \cref{eq_loss} is over the distribution $\mu$ that generates the training set; this assumption can be used to derive generalization guarantees \cite{pac,smale}. This is not always the case: under \emph{distribution shift}, the distribution over the test points is $\mu' \neq \mu$. A particular way the distribution can change, called covariate shift, keeps the conditional distribution $(Y|X)$ fixed but changes the marginal distribution of $X$ \cite{SHIMODAIRA2000227}.

It is important to appreciate that the term in \eqref{eq_loss} is a random variable, since $f(X; T)$ depends on the training set $T$ and any other randomness used in the algorithm $\cala$. For example, we might be interested in $\cala$'s average performance with respect to the training set, $\E_{T,X,Y} [L(Y, f(X; T))]$.

The bias-variance decomposition is a foundational way to decompose the expected loss to understand different sources of error. Consider the case where the loss function is the Euclidean squared error $L(x,y) \deq \norm{x-y}^2$ for $x,y\in\R^d$. When a prediction $f(x; T)$ at input $x$ depends only on the training set $T$, we can write
\begin{align}
     \E_{T,Y} \norm{f(x; T) - Y }^2 =&~  \E_Y \norm{Y - \E_Y Y }^2 &&\textit{(Bayes error)}\notag  \\ 
     &+ \norm{ \E_Y Y - \E_T f(x; T) }^2 && \textit{(Bias)} \label{eq_bv_decomp}\\
     &+ \E_T \norm{f(x; T) - \E_T f(x; T) }^2. && \textit{(Variance)} \notag
\end{align}

\begin{figure}
    \centering
    \begin{subfigure}{.24\textwidth}
    \includegraphics[width=\textwidth]{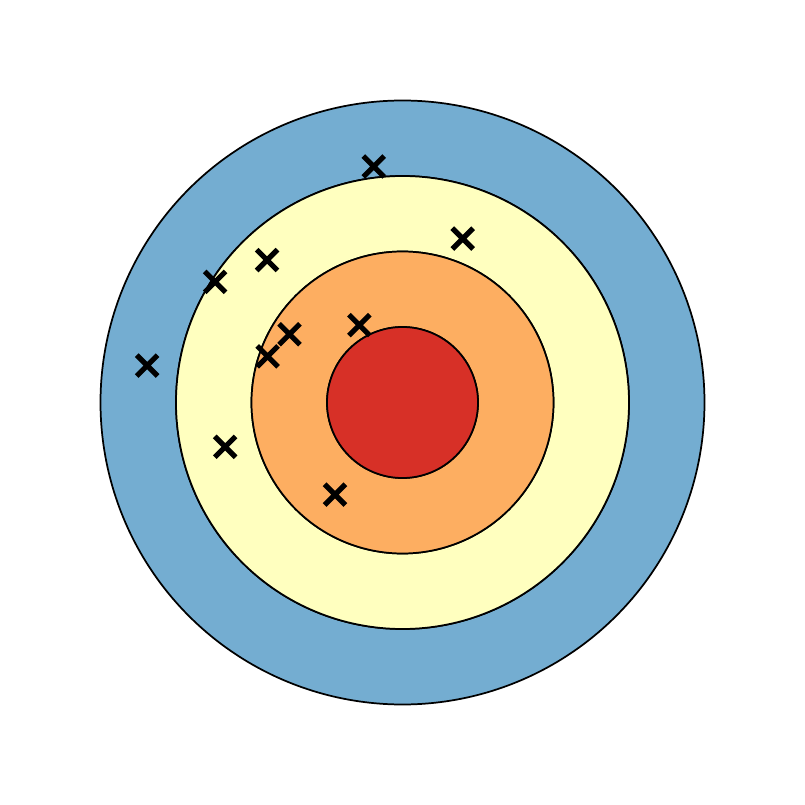}
    \caption{High bias and variance}
    \end{subfigure}
    \begin{subfigure}{.24\textwidth}
    \includegraphics[width=\textwidth]{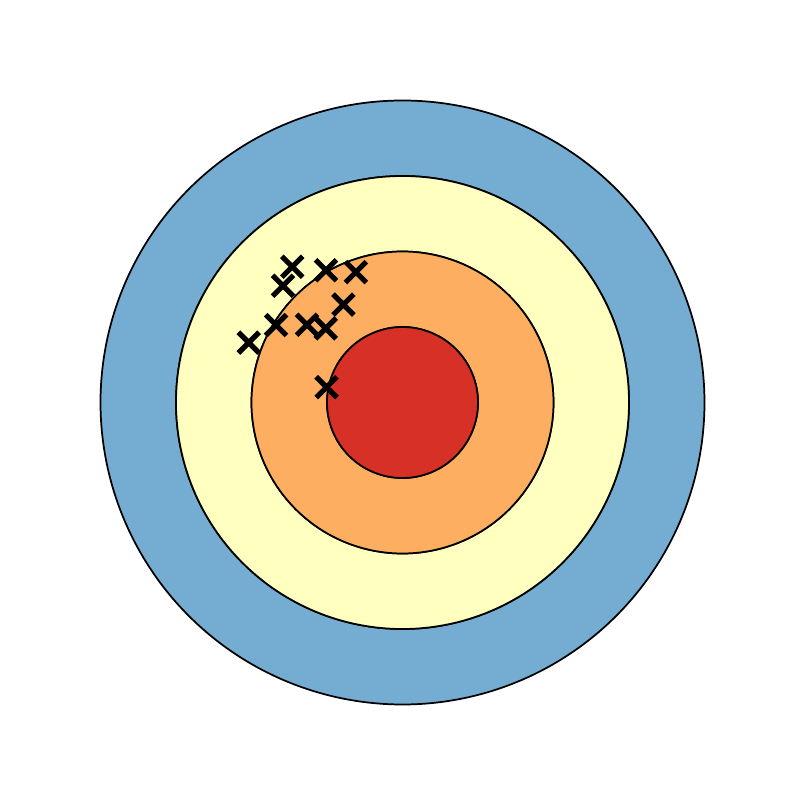}
    \caption{High bias}
    \end{subfigure}
    \begin{subfigure}{.24\textwidth}
    \includegraphics[width=\textwidth]{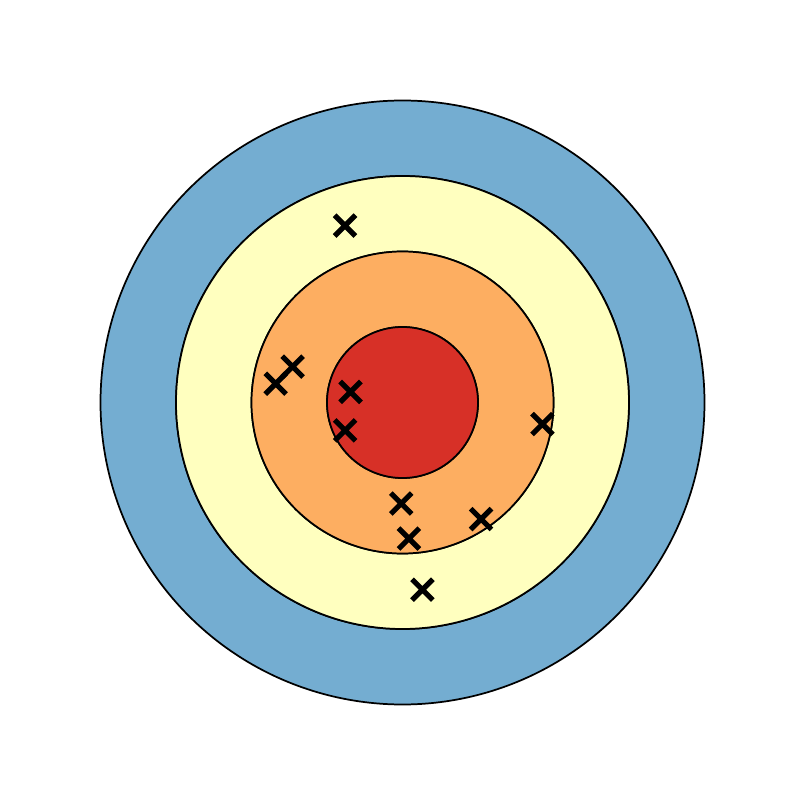}
    \caption{High variance}
    \end{subfigure}
    \begin{subfigure}{.24\textwidth}
    \includegraphics[width=\textwidth]{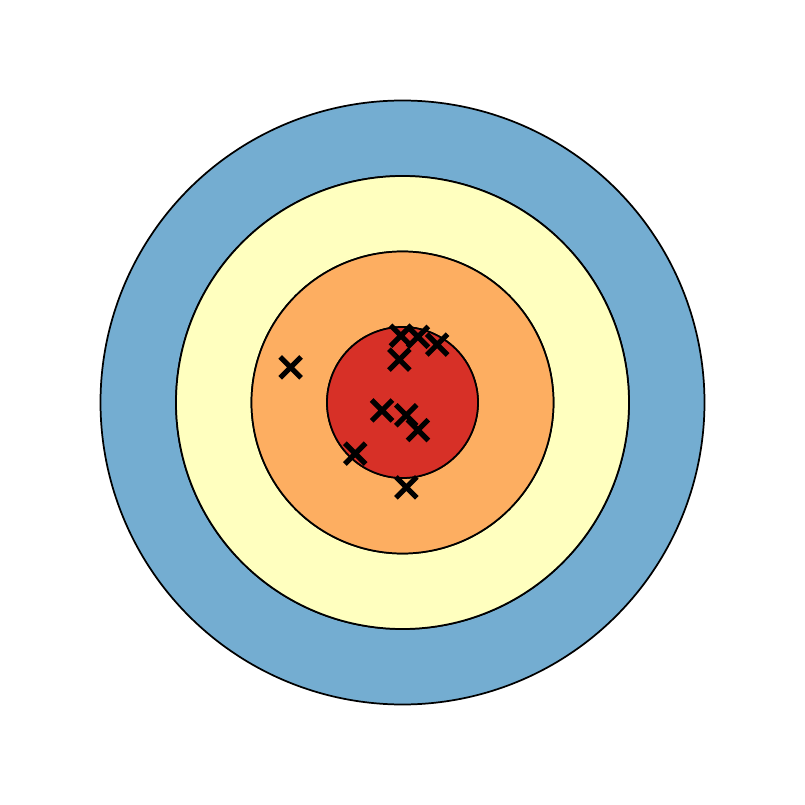}
    \caption{Low bias and variance}
    \end{subfigure}
    \caption{An illustration of the bias and variance terms for a deterministic label  defined by the bullseye $Y=(0,0)$, using draws of normal distributions with different means and standard deviations.}
    \vspace{-1em}
    \label{fig:darts}
\end{figure}


The decomposition in \cref{eq_bv_decomp} is specific to the Euclidean squared error. This specificity appears in two ways: firstly, all three terms are defined using Euclidean distances; secondly, the central terms $\E Y$ and $\E f(X; T)$ of \cref{eq_bv_decomp} are all defined by means of random variables. Nonetheless, the above interpretation of the sources of error is more general, and decomposition~\eqref{eq_bv_decomp} has been extended to broader classes of loss functions~\citep{generalized_bvd, Domingos00aunified}.

Building upon~\citet{generalized_bvd}, we consider the bias-variance decomposition for Bregman divergence losses; this class of functions includes standard losses in machine learning applications such as the KL divergence and Mahalanobis distances. Bregman divergences are not necessarily symmetric; as a consequence, bias-variance decompositions must treat the sources of randomness in the labels $Y$ differently than those in the predictions $f(X)$. In particular, the central term for predictions takes the less approachable form $\argmin_z \E_X L(z, f(X))$. 

The opaque form taken by the central prediction has severely limited the theoretical analysis of models for which the Euclidean squared loss is ill-suited, including the entirety of models used for classification tasks. In this work, we show that manipulating the central prediction is trivialized by reparameterizing the central prediction as the \emph{primal form} of the expected prediction in a dual space defined by the loss of interest. Using this reformulation, we characterize the behavior of bias and variance terms for arbitrary Bregman divergences, aiming to provide a unified understanding of how sources of randomness affect  ML algorithms.

Albeit motivated by machine learning, the bias variance decomposition applies to any pair of random variables. To emphasize this, we will simply use $X$ rather than $f(X)$ to denote predictions. $X$ is a random variable that may on other variables (such as the training set $T$). As nonsymmetric losses $L(Y, X)$ affect variables $X$ and $Y$ differently, we will continue to refer to $Y$ as the \emph{label} and to $X$ as the \emph{prediction}.  

\subsection{Related work}
\citet{generalized_bvd} generalized the bias variance tradeoff first identified for the Euclidean loss in~\citep{geman1992} to any Bregman divergence loss function $L$. For a loss function $L(Y, X)$ measuring the loss between two random variables $Y$ and $X$, \citep{generalized_bvd} showed that the corresponding biases and variances exist with respect to a ``central label'', $y = \E Y$, and a ``central prediction'', $x = \argmin \E L(x, X)$. Interestingly, the specific bias, variance, and irreducible error terms for Bregman divergence functions defined in~\citet{generalized_bvd} were also identified in~\citep{Domingos00aunified} for the 0-1 loss, despite the 0-1 loss \emph{not} being a Bregman divergence.

The Bregman representative defined in~\citet{bregman-clustering} is closely related to the ``central'' label defined by~\citep{generalized_bvd}; in particular,~\citep[Theorem 1]{bregman-clustering} is a special case of the law of total variance in label space. More generally, Bregman divergences and operations in their associated dual space are instrumental to optimization techniques such as mirror descent and dual averaging~\citep{nemirovski1983problem,DBLP:journals/mp/Nesterov09,juditsky2021unifying}.

Finally, the bias-variance tradeoff has been a fruitful venue for understanding the behavior of modern ML models~\citep{adlam2020,dascoli2020double,neal2019modern}. Due to the prevalence of the KL divergence in classification tasks, it is one of the few Bregman divergences for which the bias variance tradeoff has been specifically analyzed (\eg, \citep{yang2020rethinking}). The fact that the bias remains unchanged when averaging predictions in log-probability space has been mentioned briefly in~\citep{diettrich}, and log-probability averaging has been studied in many works, including~\citep{Brofos2019ABD,webb2020ensemble}.

\subsection{Contributions}
We build upon the decomposition by \citet{generalized_bvd}, which we reformulate by viewing the central prediction as the mean of a random variable in dual space. From this reformulation, we derive several results of interest. 

\begin{itemize}[leftmargin=*]
    \item The label and prediction variance terms satisfy a generalized law of total variance, which can be used to separate contributions of different sources of randomness. 
    \item Estimates of the bias and variance that condition on an external source of randomness overestimate the total bias and underestimate the total variance by a fixed quantity that admits a closed form.
    \item There exists a closed-form operation to aggregate predictions which does not affect the bias and reduces the variance, allowing us to recover the behavior of model ensembling under the squared Euclidean error.
\end{itemize}

\section{Background on Bregman divergences}
We begin with some background material regarding Bregman divergences, which also serves to set our notation. We refer the reader to, \eg, \citep[\S 11.2]{cesa2006prediction} for more background and related Bregman divergences. 



\begin{definition} 
Let $\mathcal X$ be a closed, convex subset of $\mathbb R^d$. A function $D: \mathcal X \times \mathcal X \to \mathbb R$ is a Bregman divergence if there exists a strictly convex, differentiable function $F$ such that 
\begin{equation}
\label{eq:bd}
    D_F\bdx{y}{x} \deq F(y) - F(x) - \scal{\nabla F(x)}{y-x}.
\end{equation}
\end{definition}
When a divergence's originating function $F$ is relevant, we will write the divergence as $D_F$.

\subsection{General properties}
Divergences are a generalization of the notion of distance, similar but less restrictive than metrics: 
\begin{proposition}[Nonnegativity]
$\forall x, y \in \mathcal X, \bd y x \geq 0$.
\end{proposition}
\begin{proposition}[Identity of indiscernibles]
$\forall x, y \in \mathcal X, \bd y x =0 \iff x=y.$
\end{proposition}
Contrary to metrics, Bregman divergences are not required to be symmetric. Rather than the triangle inequality, they follow a generalized triangle inequality where the third term can  be either positive or negative.
\begin{proposition}
  For any $x, y, z \in \mathcal X$, we have $\bd x z = \bd x y + \bd y z + \scal{\nabla F(y) - F(z)}{x-y}$.
\end{proposition}

\begin{figure}
    \centering
    \begin{subfigure}{.3\linewidth}
    \centering
      \includegraphics[width=.7\textwidth]{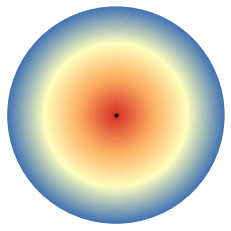}
      \caption{$F=\|x\|^2$, center at $(0, 0)$.\\\hphantom{a}}
      \label{fig:div_euclidean}
    \end{subfigure}\quad
    \begin{subfigure}{.3\linewidth}
    \centering
    \includegraphics[width=.7\textwidth]{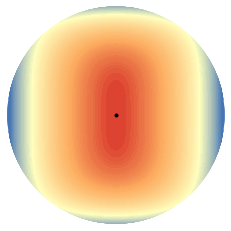}
      \caption{$F\!=\!-\log(1-x_0^2) -\log(1-x_1^4)$,\\\hphantom{(b)} center at $(0, 0)$.}
      \label{fig:div_custom}
    \end{subfigure}\quad
    \begin{subfigure}{.3\linewidth}
    \centering
    \includegraphics[width=.7\textwidth]{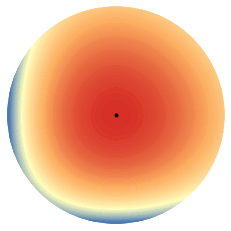}
      \caption{$F=x_0\log x_0 + x_1\log x_1$,\\\hphantom{(c)} center at $(0.5, 0.5)$.}
      \label{fig:div_entropy}
    \end{subfigure}
    \caption{Distance to the center over a disk for Bregman divergences of different convex functions $F$; the red-blue gradient indicates the corresponding distance to the center, from red (closest) to blue (furthest).}
    \label{fig_divergences}
\end{figure}

\begin{example}
  For any positive semi-definite matrix $A \in \mathbb R^{d \times d}$, the squared Mahalanobis distance $(x, y) \to (x-y)^\top A (x-y)$ is the Bregman divergence of $F(x) = x^\top Ax$.
\end{example}
\begin{example}
  The Kullback-Leibler divergence over the probability simplex $\{x \in \mathbb R^d_+ \mid \sum_i x_i = 1\}$, $(x, y) \to \sum_i x_i \log (y_i/x_i)$, is the Bregman divergence of $F(x) = \sum_i x_i \log x_i$.
\end{example}

In \cref{fig:div_euclidean}, the coloring indicates the Euclidean distance to $(0, 0)$. This distance is symmetric: the distance of any point $p=(x,y)$ to 0 is also the distance of $0$ to $p$. When we measure distances with non-symmetric functions over the disk, we obtain measures of distance which are no longer symmetric (\cref{fig:div_custom,fig:div_entropy}).

\subsection{Bregman divergences and convexity}
Many important properties of Bregman divergences are due to the convexity of their originating function $F$. It is easy to verify that any Bregman divergence $D$ is convex in its first variable.
\begin{proposition}[Convexity in the first variable]
\label{prop:convexity}
It follows from the convexity of $F$ that a divergence $D_F$ is convex in its first variable, although not necessarily in its second~\citep{bregman-convex}.
\end{proposition}

The Bregman divergence of the convex conjugate $F^*$ of $F$ will also be of particular importance. We recall that the convex conjugate of a convex function $F$ is defined as
\begin{equation*}
    \label{eq:convex-conjugate}
    F^*(x^*) = \sup_{x \in \mathcal X} \scal{x^*}x - F(x).
\end{equation*}
For any convex and differentiable function $F$, we recall the following properties, writing $x^* \deq \nabla F(x)$. 
\begin{proposition}
\label{prop:prereq2}
  $(F^*)^* = F$, and for any $x \in \mathcal X$, $x = (x^*)^*  = \nabla F^*(\nabla F(x))$. 
\end{proposition}

\begin{proposition}
\label{prop:bregman-dual}
  For any Bregman divergence $D_F$ and  $x, y \in \mathcal X$, $D_F\bdx x y = D_{F^*}\bdx {y^*} {x^*}$.
\end{proposition}
 See \cite[Chapter 11]{cesa2006prediction} for a discussion and the proofs of the propositions \ref{prop:prereq2} and \ref{prop:bregman-dual}.
\begin{proposition}
  \label{prop:minimizers}
  Let $X$ be a random variable over $\mathcal X$, and $D_F$ be any Bregman divergence over $\mathcal X  \times \mathcal X$. The minimizers of the expected divergence from and to $X$ satisfy the following equalities:
  \begin{enumerate}[label=(\roman*)]
      \item $z = \argmin_{z \in \mathcal X} \E_X D_F\bdx X z \iff z = \E X$\quad (c.f. Prop. 1 of \citep{bregman-clustering}). \label{eq:mean}
      \item $z = \argmin_{z \in \mathcal X} \E_X D_F \bdx z X \iff \nabla F(z) = \E[\nabla F(X)]$\quad (c.f. Lemma 1 of~\citep{generalized_bvd}).\label{eq:dual-mean}
  \end{enumerate}
\end{proposition}

The minimizers $\ref{eq:mean}$ and $\ref{eq:dual-mean}$ of the expected divergence \emph{to} and from a random variable will be the central terms necessary to define bias-variance decompositions for arbitrary Bregman divergences.


\section{Bias-variance decomposition}
 The main obstacle in bias-variance decompositions for Bregman divergences lies in the form of the minimizer \emph{to} a random variable, $\argmin_z \E \bd z X$, which is less easily manipulated than the minimizer \emph{from} a random variable, $\argmin \E \bd Z z = \E X$. However, it turns out that $z \to \argmin_z \E D_F\bdx z X$ can also be formulated as a mean, once we consider the dual space defined by the convex conjugate of $F$:
 \[\argmin_{z \in \mathcal X} \E D_F\bdx z X = \argmin_{z \in \mathcal X} \E D_{F^*} \bdx {X^*} {z^*} = \left(\argmin_{z^* \in \mathcal X^*} D_{F^*}\bdx {X^*} {z^*}\right)^* = (\E X^*)^*.\]
 
This reformulation of the central prediction is crucial to our analysis, and, to the extent of our knowledge, novel. We introduce this form of the expected divergence to a random variable in the definition below.
\begin{definition}
  \label{def:dm}
  For a random variable $X$ over $\mathcal X$ and a Bregman divergence $D_F$ over $\mathcal X \times \mathcal X$, we define the \emph{dual mean} as the primal form of the mean of $X$ taken in dual space: $\dm X \deq (\E X^*)^* = \argmin_z \E \bd z X$.
\end{definition}
\begin{remark}
\label{rem:sym}
When $D$ is symmetric, $\E X = \dm X$ for any random variable $X$ over $\mathcal X$.
\end{remark}

\subsection{General statement}
Bias-variance decompositions reduce to breaking up the average gap between two independent random variables, $\E \bd Y X$. 
For any Bregman divergence $D$, we know from~\citet{generalized_bvd} that the average loss $\E \bd Y X$ can be decomposed into three terms:
\begin{equation}
    \label{eq:bv}
    \E \bd Y X = \underbrace{\E \bd Y {\E Y}}_{\V Y} + \bd{\E Y} {\dm X} + \underbrace{\E \bd {\dm X}X}_{\dv X}.
\end{equation}
The first and third terms on the right-hand side of \cref{eq:bv} might both be interpreted as variances. Indeed, when $D$ is the Euclidean distance, both variance terms in \cref{eq:bv} take the classical forms $\V Y = \E(Y - \E Y)^2$ (Bayes error) and $\dv X = \E (X - \E X)^2$, as a consequence of Remark~\ref{rem:sym}. 

However, we insist once again upon the fact that in the general case, the order of $X$ and $Y$ affects which notion of variance is considered for each random variable: (a)  the ordering of the central term and its corresponding random variable is specific to each variance term; (b) the central term is either the mean $\E Y$ or the dual mean $\dm X$. To emphasize (b), we will refer to $\V Y$ as the \emph{primal variance}, and to $\dv X$ as the \emph{dual variance}.

\subsection{Properties of the variance terms}
Despite their non-standard forms, both primal and dual variances in \eqref{eq_bv_decomp} satisfy fundamental properties associated with the standard variance $\E (X - \E X)^2$.

\begin{proposition}[Non-negativity]
    The primal and dual variances are non-negative.
\end{proposition}

\begin{proposition}[Variance of constants]
    Let $c \in \mathcal X$ be a constant; then $\V[c]=0$ and $\dv[c] = 0$. Conversely, let $X$ be a random variable over $\mathcal X$ such that $\V X = 0$ (resp. $\dv X = 0$). Then $X$ is almost surely a constant.
\end{proposition}

Of particular interest to us is the law of total variance, and whether it applies to the primal and dual variances introduced above. Recall that, given two variables $Y$ and $Z$, the law of total variance decomposes the standard (Euclidean) variance as $\V Y = \E[\V (Y \mid Z)] + \V [\E(Y \mid Z)]$: the variance of $Y$ is the sum of the variances respectively \emph{unexplained} and \emph{explained} by $Z$. 

The key result here is that  primal and dual variances introduced above  satisfy generalized formulations of the law of total variance. For the primal variance $\V$, the law of total variance remains unchanged.
\begin{restatable}{lemma}{primaltotalvariance}
  \label{lem:tv_1}
  Let $Y, Z$ be random variables over $\mathcal X$, and let $\V Y \deq \E \bd {Y}{\E Y}$ for a given Bregman divergence $D$ over $\mathcal X$. Then, 
  \[\V[Y] = \E [\V[Y|Z]] + \V [\E[Y|Z]].\]
\end{restatable}
 
Obtaining a law of total variance for the dual variance $\dv X = \E \bd{\dm X}{X}$, however, requires a slight change to accommodate the fact that the divergence is defined with respect to the mean taken in the dual space. Thankfully, the dual mean itself satisfies its own law of total expectation. Although the proof is trivial using the characterization $\dm X = (\E X^*)^*$, this result is of independent interest.
\begin{restatable}{lemma}{iteratedexpectations}
  \label{prop:iterated-expectations}
  Let $X,Z$ be random variables on $\mathcal X$. Then $\dm X = \dm_Z [\dm [X | Z]]$.
\end{restatable}
\begin{restatable}{lemma}{dualtotalvariance}
\label{lem:tv_2}
Let $X$ and $Z$ be a random variables over $\mathcal X$, and define as above $\dv X \deq \E \bd{\dm X}{X}$. Then, 
  \[\dv[X] =  \E[\dv [X | Z]] + \dv[\dm [X|Z]].\]
\end{restatable}

Additionally to further characterizing $\V Y$ and $\dv X$ as variances, ~\cref{lem:tv_1,lem:tv_2}  can be used to separate the contribution of different sources of randomness to the variance term, using the same technique as~\citep{adlam2020}.

\section{Conditional bias-variance tradeoff}
Although the bias-variance tradeoff in \cref{eq:bv} applies to any source of randomness, it is important to understand how any empirical analysis of the bias and variance on real-world models is affected by the implicit conditioning on random variables that we cannot sample from arbitrarily. 

A common example is the randomness due to the choice of training set: most ML benchmarks only provide one training set. Although the bias and variance due to an algorithm's innate stochasticity (\eg, random seed) can be estimated, the contribution of the training set's stochasticity cannot be folded in to the expectations that define the bias and variance terms. Our goal here is to understand how having only one sample of a given random variable biases our estimates of bias and variance.

\subsection{Conditioning in prediction space}
Let $X, Z$ be two random variables over $\mathcal X$, and write $(X \mid z)$ the random variable $X$ conditioned on a given value of $z$ of $Z$. We can write the decomposition of \cref{eq_bv_decomp} for $(X \mid z)$ as follows (for simplicity, we assume here that there is no randomness in the label $Y=y$):
\begin{equation}
    \label{eq:bv_on_z}
    \E_{X\mid z} \bd y X = \bd {y}{\dm (X \mid z)} + \E_{X \mid z} \bd{\dm (X \mid z)}{X}.
\end{equation}
Taking expectations over $Z$ on both sides of \cref{eq:bv_on_z} yields a \emph{conditional} bias variance decomposition:
\begin{equation}
\label{eq:conditional-bv-x}
    \E \bd y X = \underbrace{\E_Z \bd {y}{\dm (X \mid Z)}\vphantom{\Big]}}_{\textup{conditional bias}} + \underbrace{\E_Z \E_{X \mid Z}\Big[\bd{\dm (X \mid z)}{X}\Big| Z\Big]}_{\textup{conditional variance}}.
\end{equation}
Assume that we only get one draw of $Z$ (for example, we only get to estimate the performance of our learning algorithm on a single training set), but that we can sample as many times from $(X \mid Z)$ as we wish (we can train on the training set with different random seeds and evaluate the result). Then, the bias and variance we estimate are those of \cref{eq:conditional-bv-x}, where expectations over $Z$ are obtained by a single sample.

How incorrect are these estimates of the true bias and variance, which incorporate randomness over $Z$?

\begin{restatable}{proposition}{propconditionalsecond}
\label{prop:conditional}
  Let $X, Z$ be two random variables over $\mathcal X$. The conditional bias (resp. variance) of \eqref{eq:conditional-bv-x} overestimates (resp. underestimates) their respective \emph{unconditional} values by the quantity $\E_Z \bd{\dm X}{\dm (X|Z)}$.
  \begin{align*}
      \E_Z \bd{y}{\dm (X \mid Z)} &= \bd y {\dm X} + \E_Z \bd{\dm X}{\dm (X\mid Z)} && \text{(Bias is overestimated)}\\
      \E_Z \E_{X \mid Z}\Big[ \bd{\dm (X \mid z)}{X}\Big| Z\Big] &= \E \bd{\dm X} X - \E_Z \bd{\dm X}{\dm (X\mid Z)} && \text{(Variance is underestimated)}.
  \end{align*}
\end{restatable}
As we are guaranteed to have $\E_Z \bd{\dm X}{\dm (X|Z)} \ge 0$, the bias will always be overestimated, and the dual variance underestimated, by their conditional estimates. 

\subsection{Conditioning in label space}
We can apply a similar reasoning when conditioning on a random variable affecting the labels, assuming for simplicity that there is no randomness in the prediction $X=x$. Letting $Y, Z$ be two random variables, the loss decomposes as a conditional bias and variance as follows:
\begin{equation}
    \label{eq:conditional-bv-y}
    \E_Y \bd Y x = \E_Z\E_{Y \mid Z} [\bd Y x \mid Z] = \underbrace{\E_Z \bd{\E (Y\mid Z)}{x}}_{\textup{Conditional bias}} + \underbrace{\E_Z\E_{Y \mid Z} \Big[\bd Y {\E (Y\mid Z)}\Big|Z\Big]}_{\textup{Conditional variance}}.    
\end{equation}

\begin{restatable}{proposition}{propconditionalfirst}
\label{prop:conditional-first}
  Let $Y, Z$ be two random variables over $\mathcal Y$. The conditional bias (resp. variance) in \eqref{eq:conditional-bv-y} overestimates (resp. underestimates) their respective total values by the quantity $\E_Z \bd{\E(Y|Z)}{\E Y}$.
  \begin{align*}
      \E_Z \bd{\E (Y\mid Z)}{x} &= \bd{\E Y}{x} + \E_Z \bd{\E(Y\mid Z)}{\E Y} \\
      \E_Z\E_{Y \mid Z} \Big[\bd{Y}{\E (Y|Z)}\Big|Z\Big] &= \E \bd{Y}{\E Y} - \E_Z \bd{\E(Y\mid Z)}{\E Y}.
  \end{align*}
\end{restatable}

\section{Averaging in primal and dual spaces}
Motivated by prior work which showed that the bias-variance trade-off is a fruitful avenue to understand \emph{ensembles} of predictions~\cite{dascoli2020double, adlam2020, neal2019a}, we conclude this paper by analyzing the implications of the existence of the dual space for for convex combinations of either predictions or labels. 

In modern machine learning, ensembling typically operates by averaging the output of identical (deep) models trained with different random seeds. Following our notation, this amounts to replacing a prediction $X$ by an average prediction $\hat X = \frac 1n \sum_i X_i$, where each $X_i$ is drawn in \iid fashion.  

We can characterize some desirable properties of any averaging technique by mapping to the expected behavior of the empirical mean under the squared Euclidean loss:
\begin{enumerate}[label=(P\arabic*)]
    \item \textbf{Reduced variance}: averaging \iid random variables must not increase the variance. \label{var-reduction}
    \item \textbf{Constant bias}: averaging \iid random variables must leave the bias unchanged. \label{bias-unchanged}
\end{enumerate}
For the bias-variance tradeof in \eqref{eq:bv}, we will see that maintaining properties \ref{var-reduction} and \ref{bias-unchanged} in prediction space requires departing from the empirical mean as averaging operator.

\subsection{Primal averaging}

We begin by analyzing the ensembling operation described above; we will call this operation \emph{primal averaging}, for reasons that will be clear momentarily.
\begin{definition}[Primal averaging]
  Let $y_1, \ldots, y_n \in \mathcal X$. The \emph{primal average} of the $y_i$ is $\hat y = \frac 1n \sum_i y_i$. 
\end{definition}
The linearity of the mean and the convexity of any Bregman divergence in its first variable suffice easily to show that primal averaging in label space satisfies the desired properties.
\begin{restatable}{proposition}{primal-primal-space}\textup{(Primal averaging in label space)}.
Primal averaging in label space leaves the bias unchanged, and reduces the primal variance: 
\begin{align*}
    \bd{\E \hat Y}{\dm X} &= \bd{\E Y}{\dm X} \\
    \V \hat Y &\le \V Y.
\end{align*} 
\end{restatable}
Things are less straightforward in prediction space. Nonetheless, we can show under some simple assumptions that primal averaging reduces the prediction variance $\E \bd{\dm X}{X}$.
\begin{restatable}{proposition}{propvanillavariance}
  \label{prop:vanilla-variance}
  Let $D$ be a Bregman divergence that is \emph{jointly} convex in both variables. Let $X_1, \ldots, X_n$ be $n$ random variables over $\mathcal X$ drawn in \iid fashion, and define $\hat X = \frac 1n \sum_i X_i$. Then, the dual variance $\dv$ under divergence $D$ is reduced by primal averaging:
  \[\dv \hat X \le \dv X.\]
\end{restatable}

Maintaining a constant bias requires a much more stringent assumption: the divergence $D$ must be symmetric, ensuring that $\dm$ and $\E$ are equivalent operators. In the general case, however, primal averaging in the space $\mathcal X$ of predictions can either increase or decrease the bias $\bd{\E Y}{\dm X}$.

\begin{restatable}{proposition}{propbias}
  \label{prop:bias}
  Let $D$ be the KL divergence. There exists a distribution $\mathcal D$ over predictions $X \in \mathbb R^2$ and a label $y \in \{0, 1\}$ such that the divergence $D(y \| \dm X)$ satisfies
  \begin{align*}
      \bd{y}{\dm \hat X} &< \bd{y}{\dm X} \\
      \bd{1-y}{\dm \hat X} &> \bd{1-y}{\dm X},
  \end{align*}
 where as above, we define the random variable for ensemble predictions $\hat X = \frac 1n \sum_i X_i$ and by abuse of notation, we conflate $y \in \{0, 1\}$ with its one-hot vector representation $(y = 0, y=1) \in \mathbb R^2$.
\end{restatable}

\subsection{Dual averaging}
That primal averaging doesn't preserve the bias (and can, in fact, increase it!) is a strong departure from what one might expect. It is natural to seek an averaging method over prediction space that would maintain both \ref{var-reduction} and \ref{bias-unchanged} without requiring assumptions on the convexity or symmetry of the divergence $D$.

Once again, our path forward is guided by the dual expectation $\dm$. Intuitively, we need an averaging technique such that the average predictor $\bar X$ satisfies $\dm \bar X = \dm X$; the following definition satisfies this requirement.  

\begin{definition}[Dual averaging]
    Let $z_1, \ldots, z_n \in \mathcal Z$. The \emph{dual average} of the $z_i$ is $\bar z = (\frac 1n \sum_i z_i^*)^*$. 
\end{definition}
Similarly to the dual mean (\cref{def:dm}), the dual average is the primal form of an operation taken in dual space.

\begin{restatable}{proposition}{propdualbias}
\label{prop:dual-bias}
\textup{(Dual averaging in prediction space)}. Dual averaging in the space $\mathcal X$ of models leaves the bias unchanged and reduces the dual variance:
\begin{align*}
    \bd{\E Y}{\dm \bar X} &= \bd{\E Y}{\dm X} \\
    \dv \bar X &\le \dv X.
\end{align*}
\end{restatable}
Note that, in contrast to the reduction in variance that occurs for vanilla ensembles (\cref{prop:bias}), we also no longer require that $D$ be jointly convex; the natural convexity of $D$ in its first argument is sufficient.

Naturally, one might ask how dual averaging in primal space affects the bias and variance. Recall (\cref{prop:bregman-dual}) that $D_F\bdx y x = D_{F^*}\bdx {x^*}{y^*}$. Thus, we can rewrite the primal variance in the form of a dual variance:
\[\V Y \deq \E D_F\bdx{Y}{\E Y} = \E D_{F^*}\bdx{(\E Y)^*}{Y^*} = \E D_{F^*}\bdx{\dm Y^*}{Y^*} = \dv_{F^*} Y^*.\]

Similarly, we can rewrite the bias $D_F \bdx{\E Y}{\dm X}$ as $D_{F^*}\bdx {\E X^*}{\dm (Y^*)}$. Since dual averaging is equivalent to primal averaging \emph{in dual space}, the previous two equalities are sufficient to obtain equivalent results to those of Propositions~\ref{prop:vanilla-variance} and \ref{prop:bias}.

\begin{restatable}{proposition}{propdual-predictionspace}
\label{prop:dual-prediction-space}
\textup{(Dual averaging in label space).} Dual averaging in label space can either increase or decrease the bias, but will reduce the primal variance if $D$ is jointly convex.
\end{restatable}

\section{Conclusion}
Given a loss function between a prediction $X$ and a label $Y$ (both random variables), a bias-variance decomposition splits the expected loss into three terms. The \emph{label noise} is the expected distance from the label to its expected value; the \emph{model variance} is the expected distance from the prediction to a corresponding ``central prediction''; finally, the \emph{bias} is the distance between the expected label and the central prediction. Initially described in~\citep{geman1992}, this decomposition was subsequently generalized to arbitrary Bregman divergences~\citep{generalized_bvd}.

For the Euclidean squared loss, the central prediction that appears in the bias and model variance terms is simply the expected prediction $\E X$; yet, the bias-decomposition has remained opaque for other Bregman divergences. This can in no small part be attributed to the less tractable form $\argmin_z \E D[z \| X]$ taken by the central prediction in the general case. Unfortunately, this limitation has impeded the application of the bias-variance decomposition to popular losses in machine learning, including the cross-entropy loss overwhelmingly used in classification tasks.

In this work, we show that the complexities of the generalized bias-variance decomposition are entirely resolved by analyzing the decomposition in a dual space defined by the loss function. From this new perspective, the central prediction is simply the primal form of the expected prediction in dual space. For symmetric Bregman divergences, the primal and dual spaces are one and the same: we recover the central prediction as the expected prediction for symmetric losses, including the Euclidean squared error. 

This reformulation of the central prediction allows us in turn to show that the model variance satisfies crucial properties. In particular, the model variance follows a generalized law of total variance, allowing for precise analyses of different sources of error. We subsequently isolate the irreducible bias of conditional estimates for the decomposition's bias and variance terms.

Finally, the dual perspective on the bias-variance tradeoff provides a straightforward framework within which to analyze the behavior of ensembles of predictors. We show that although averaging predictions in primal space (which amounts to simply taking the empirical mean of different predictions) will reduce the variance under gentle assumptions, primal averaging can have arbitrary effects on the bias. Conversely, averaging predictions in dual space will always reduce the variance and leave the bias unchanged, recovering the known behavior of classical ensembling under the Euclidean squared loss.

\newpage
\bibliography{refs}
\bibliographystyle{zelda_bst}
\clearpage
\appendix
\section{Proofs}
\label{app:proofs}
\subsection{Bias-variance decomposition}

\primaltotalvariance*
\begin{proof} By the generalized triangle inequality for Bregman divergences, we have
  \begin{align*}
      \V Y =&~ \E \bd Y {\E Y} \\
      =&~ \E \Big[\bd[\Big]{Y}{\E[Y|Z]} + \bd[\Big]{\E[Y|Z]}{\E[Y]} + \scal{\nabla F(\E[Y|Z]) - \nabla F(\E Y)}{Y - \E[Y|Z]}\Big] \\
      =&~ \E_Z \Big[\E_{Y|Z} \bd[\Big]{Y}{\E[Y|Z]} \Big| Z\Big] + \E_Z \Big[\bd[\Big]{\E[Y|Z]}{\E_Z \E_{Y|Z}[Y]} \Big| Z\Big] \\
      &+ \E_Z \Big[\scal{\nabla F(\E[Y|Z]) - \nabla F(\E Y)}{\E[Y|Z] - \E[Y|Z]}\Big | Z \Big] \\
      =&~ \E_Z [\V[Y|Z]] + \V[\E[Y|Z]]
  \end{align*}
  where RHS of the scalar product is equal to 0 by the law of total expectation.
\end{proof}

\iteratedexpectations*
\begin{proof}
This result is a straightforward consequence of the standard law of iterated expectation and of the characterization of the dual mean as $\dm X = (\E X^*)^*$:
\[\dm_Z (\dm_{X|Z} X) = \dm_Z\Big[(\E_{X|Z}X^*)^*\Big] = (\E_Z \E_{X|Z} X^*)^* = (\E X^*)^* = \dm X.\]
\end{proof}

\dualtotalvariance*
\begin{proof}
By law of iterated expectations, we have
  \begin{align*}
      \dv X &= \E_Z \E_{X|Z} [\bd{\dm X}{X} \mid Z] \\
             &\overset{(a)}{=} \E_Z \E_{X|Z} \Bigg[ \bd{\dm X}{\dm (X|Z)} + \bd{\dm (X|Z)}{X} \\
             &~~~~~+\Big\langle \dm(X|Z) - \dm X \mid \nabla F(X) - \nabla F(\dm (X|Z)) \Big\rangle \Bigg| Z\Bigg]\\
             &= \E_Z \bd{\dm X}{\dm (X|Z)} + \E_Z\underbrace{\E_{X|Z} \Big[\bd{\dm (X|Z)}{X}\Big| Z\Big]}_{\dv [X \mid Z]}\\
             &~~~~~+\E_Z \Big\langle \dm(X|Z) - \dm X \mid \underbrace{\E_{X|Z}\nabla F(X) - \nabla F(\dm (X|Z))}_{=0 \textup{ by Prop.~\ref{prop:minimizers} $\ref{eq:dual-mean}$}} \Big\rangle \\
             &\overset{(b)}{=} \E_Z \bd{\dm_Z [\dm(X|Z)]}{\dm (X|Z)}+ \E_Y [\dv[X|Z]] \\
             &= \dv [\dm(X|Z)] + \E_Z [\dv X|Z].
  \end{align*}
  where $(a)$ follows from the generalized triangle inequality for Bregman divergences, and $(b)$ is the law of iterated expectations for $\dm$ (lemma~\ref{prop:iterated-expectations}).
\end{proof}

\subsection{Conditional bias-variance tradeoff}
\propconditionalsecond*
\begin{proof}


Applying \eqref{eq:bv} to the conditional bias $\E_Z D[y \| \dm (X|Z)]$, we have
  \begin{align*}
      \E_Z \bd{y}{\dm (X\mid Z)} &= \bd[\Big] {y}{\dm [\dm (X|Z)} + \E_Z \bd[\Big]{\dm[\dm (X|Z)]}{\dm (X|Z)} \\
      &= \bd {y}{\dm X} + \E_Z \bd{\dm X}{\dm (X|Z)},
  \end{align*}
where the last equality stems from the law of iterated expectations for $\dm$, showing the equality for the bias terms. The result for the variance terms follows immediately, as conditional bias and variance have the same sum as the full bias and variance.
\end{proof}

\propconditionalfirst*
\begin{proof}


Applying \eqref{eq:bv} to the conditional bias $\E_Z D[\E (Y|Z) \| x]$, we have
  \begin{align*}
      \E_Z\bd{\E (Y\mid Z)}{x} &= \bd[\Big]{\E [\E (Y\mid Z)]}{x} + \E_Z \bd[\Big]{\E (Y \mid Z)}{\E[\E (Y \mid Z)]} \\
      &= \bd{\E Y}{x} + \E_Z \bd{\E (Y\mid Z)}{\E Y},
  \end{align*}
where the last equality stems from the law of iterated expectations for $\E$, showing the equality for the bias terms. The result for the variance terms follows immediately, as conditional bias and variance have the same sum as the full bias and variance.
\end{proof}

\subsection{Averaging in primal and dual spaces}
\propvanillavariance*

\begin{proof} Let $D: S \times S \to \mathbb R^+$ be a Bregman divergence jointly convex in both variables. Let $\hat X = \frac 1n \sum_i X_i$, where the $X_i$ are \iid. By convexity, for any $z \in \mathcal X$,
\begin{align*}
    \bd{z}{\hat X} &\le \frac 1n \sum_i \bd{z}{X_i} \\
    \E \bd{z}{\hat X} &\le \frac 1n \sum_i \E \bd{z}{X_i} = \E \bd{z}{X} \\
    \min \E \bd{z}{\hat X} &\le \min_z \E \bd z X.
\end{align*}
As $\dm X = \argmin_z \E \bd z X$, it follows that $\E \bd{\dm X} X = \min_z \E \bd z X$, concluding the proof.
\end{proof}

\propbias*
\begin{proof}
  For any one-hot label $y \in \{0, 1\}$ and probability vector $x$, we have $\kl \bdx y x = \log x_y$, and $\kl \bdx {1-y} x = \log(1-x_y)$. As $x \to \log 1-x$ is decreasing, it suffices to prove that there exists a distribution $\mathcal D$ such that $\kl \bdx y {\dm \hat X} \neq \kl \bdx y {\dm X}$. In fact, it suffices to prove the existence of a distribution $\mathcal D$ such that $\dm X \neq \dm \hat X$.
  
  For the cross-entropy loss, we know\footnote{See, \eg, ~\citep{yang2020rethinking}.} that $\dm X = \textup{softmax}(e^{\E \log X})$. Let $\mathcal D$ be the distribution that assigns equal probability to $x=(0.8, 0.2)$ and $x=(0.6, 0.4)$, and is zero elsewhere. The equivalent ensemble distribution assigns $1/4$ probability to $(0.8, 0.2)$ and $(0.6, 0.4)$, and $1/2$ probability to $(0.7, 0.3)$. A simple numerical computation then shows that $\dm X \neq \dm \hat X$, concluding our proof. 
\end{proof}
\propdualbias*
\begin{proof}
To preserve bias, it suffices to have $\dm \hat X = \dm X$. By definition of $\hat X$, we have
\[\dm \hat X = \Big(\E \hat X^* \Big)^* = \Big(\E \Big[\frac 1n \nlsum_i X_i^*\Big]\Big)^* = (\E X^*)^* = \dm X.\]
We now focus on the variance. Using the fact that $D_F \bdx p q = D_{F^*} \bdx{q^*}{p^*}$, we have
\begin{align*}
    \E D_F\bdx{\dm \hat X}{\hat X} &= \E D_F \bdx{\dm X}{\hat X} \\
                               &= \E D_{F^*} \bdx{\hat X^*}{(\dm X)^*} \\    
                               &= \E D_{F^*}\bdx[\Big]{\frac 1n \nlsum_i  X_i^*}{(\dm X)^*} \\
                               &\overset{(a)}{\le} \frac 1n \nlsum_i \E D_{F^*} \bdx[\Big]{X_i^*}{(\dm X)^*} \\
                               &\le \frac 1n \nlsum_i \E D_{F}\bdx{\dm X}{X_i} \\
                               &\le D_F\bdx {\dm X}{X}.
\end{align*}
where $(a)$ follows from the convexity of $D_{F^*}$ in its first argument.
\end{proof}

\end{document}